\documentclass[11pt]{article}

\usepackage{geometry}
\usepackage{fullpage}
\usepackage{epsf}
\usepackage{fancyheadings}
\usepackage{graphics}
\usepackage{graphicx}
\usepackage{psfrag}

\usepackage{parskip}
\usepackage{algorithm}
\usepackage{algorithmic}

\usepackage{mathtools}
\usepackage{tikz-cd}
\usepackage{dcmacro}
\usepackage{comment}
\usepackage{xcite}
\usepackage{xr-hyper}

\usepackage{color}
\usepackage{amsthm}
\usepackage{amsfonts}
\usepackage{amsmath}
\usepackage{amssymb}
\usepackage{microtype}
\usepackage{subfigure}
\usepackage{booktabs}

\usepackage{natbib}
\usepackage{float}
\usepackage{hyperref}
\hypersetup{
    colorlinks,%
    citecolor=blue,%
    filecolor=blue,%
    linkcolor=blue,%
    urlcolor=blue
}

\theoremstyle{plain}

\usepackage{commath}

\newcommand{\rstd}{R^{\mathrm{std} }}
\newcommand{\termone}{T_1}
\newcommand{\termtwo}{T_2}
\newcommand{\termthree}{T_3}
\newcommand{\termfour}{T_4}
\newcommand{\RobustRisk}{R_{\mu, \Sigma}^{B, \varepsilon}}
\newcommand{\RobustOptRisk}{R_{\mu, \Sigma}^{B, \varepsilon}*}

\newcommand{\CleanOptRisk}{\rstd_{\mu', \Sigma}*}
\newcommand{\Zpop}{z_\Sigma(\mu)}
\newcommand{\Zemp}{\widehat{z}}
\newcommand{\AdvSNR}{\mathrm{AdvSNR}}
\newcommand{\StdSNR}{\mathrm{StdSNR}}
\newcommand{\defn}{:=}

\newcommand{\termfive}{U_1}
\newcommand{\termsix}{U_2}
\newcommand{\termseven}{U_3}

\begin{document}

\begin{center}

{\bf{\LARGE{Sharp Statistical Guarantees for Adversarially Robust Gaussian Classification}}}

\vspace*{.2in}

{\large{
\begin{tabular}{ccc}
Chen Dan$^\star$ & Yuting Wei$^\dagger$ & Pradeep Ravikumar$^\ddagger$\\
\end{tabular}
}}

\vspace*{.2in}

\begin{tabular}{c}
Computer Science Department$^\star$\\
Department of Statistics and Data Science$^\dagger$\\
Machine Learning Department$^\ddagger$\\
Carnegie Mellon University

\end{tabular}

\vspace*{.2in}

\today

\vspace*{.2in}
\end{center}

\begin{abstract}
	\vspace{0.1cm}
	Adversarial robustness has become a fundamental requirement in modern machine learning applications. Yet, there has been surprisingly little statistical understanding so far. In this paper, we provide the first result of the \emph{optimal} minimax guarantees for the excess risk for adversarially robust classification, under Gaussian mixture model proposed by \cite{schmidt2018adversarially}. The results are stated in terms of the \emph{Adversarial Signal-to-Noise Ratio (AdvSNR)}, which
	generalizes a similar notion for standard linear classification to the adversarial setting. For the Gaussian mixtures with AdvSNR value of $r$, we establish an excess risk lower bound of order $\Theta(e^{-(\frac{1}{8}+o(1)) r^2} \frac{d}{n})$ and design a computationally efficient estimator that achieves this optimal rate. Our results built upon minimal set of assumptions while cover a wide spectrum of adversarial perturbations including $\ell_p$ balls for any $p \ge 1$.
\end{abstract}

\section{Introduction}

Recent years, machine learning algorithms have revolutionized our life due to their tremendous success in a variety of different domains such as image classification, natural language processing and strategy games (e.g. \citet{krizhevsky2012imagenet,bahdanau2014neural,silver2016mastering}). These algorithms often achieve extremely accurate performances yet are susceptible to small perturbations of the inputs. In particular, \citet{szegedy2013intriguing} (among others e.g. \citet{goodfellow2014explaining,papernot2016limitations}) noticed that small perturbations (nearly imperceptible) to images could cause neural network classifiers to make wrong predictions with high confidence. 
While a growing amount of effort has been made in order to empirically improve the robustness of these learning algorithms against adversarial attacks,
the problems of assessing statistical optimality, understanding generalization and statistical significance are important but far less understood. In this paper, we take a step towards this end.

In this work, we consider the adversarially robust classification problem under the Gaussian mixture model proposed by \citet{schmidt2018adversarially}.  
While the classification for mixture of Gaussian distributions --- which is also referred to as discriminant analysis --- has now been standard in statistics and computer science literature (see, e.g.~\citet{mclachlan2004finite}), it is only until recently that researchers start to consider what can go wrong in the adversarial scenarios for this simple problem.  
It turns out (and as is shown in the sequel) that this simple yet instructive model demonstrates clear tradeoffs between adversarially robustness and the statistical complexities, and at the same time, capturing some of the features one would encounter in real applications.

Under minimal assumptions of the adversarial perturbations, we provide optimal minimax lower bounds, and show that a natural computationally efficient estimator achieves these minimax lower bounds in terms of the adversarial signal to noise ratio. Putting these together gives a sharp characterization of the intrinsic hardness of this problem in terms of how far one can push towards a robust estimator without any essential loss of statistical accuracy. 
These optimal lower and upper bounds are useful since that they provide a comprehensive view of the adversarially robust sample complexity of the conditional Gaussian model, which could then be contrasted with that of the rates of the classical conditional Gaussian model.

Despite of an extensive line of work considering this problem,  \citet{schmidt2018adversarially} and \citet{NIPS2019_8968} lie most closely to this paper. 
In order to obtain tight statistical characterizations of the risk, they made a number of simplifications, which thus do not directly provide answers to the minimax sample complexity of the original problem. As one main contrast, they consider the Bayesian setting where the means of the conditional Gaussians have as prior an independent standard Gaussian distribution. 
For other simplifications, \citet{schmidt2018adversarially} considered the spherical models so that the covariance is identity and also made additional simplifications such as large separation between two Gaussians and an upper bound on the noise level. These additional assumptions made it hard to compare with that of the adversary-free scenario. More detailed comparisons and discussions are provided after our main results.

\subsection{Our contributions} 
The main contributions of this paper are summarized below, all of which are built upon a careful analysis of the classification error for linear classifiers. 
\begin{itemize}
	\item We develop the first minimax lower bounds for the classification excess risk in the conditional Gaussian model, stated in Theorem~\ref{thm:main_lower_bound}. 
	In terms of the Adversarial Signal-to-Noise Ratio (AdvSNR), this excess risk scales as $\Omega_P(\exp(-(\frac{1}{8} + o(1))r^2) \frac{d}{n})$ for \mbox{$\AdvSNR = r$}, dimension $d$ and sample size $n.$
	
	\item We construct a computationally efficient estimator based on the solution of a constrained quadratic optimization problem that has excess risk of order 
	$O_P(\exp(-(\frac{1}{8} + o(1))r^2) \frac{d}{n})$. This result is given 
	in Theorem~\ref{thm:main_upper_bound}.
	Hence, the upper bound is nearly tight (up to lower order terms in $r$) with the minimax lower bound in our regime of interest in terms of AdvSNR $r$, dimension $d$ and sample size $n$ . 
	
	\item The recipe provided herein, works for a wide range of adversarial perturbations, generalizing the result by \citet{schmidt2018adversarially} who focus only on the $\ell_\infty$-type perturbations. 
	
	\item Finally, our results are built upon minimum set of assumptions, without assuming 
	strong separations between two classes, allowing for unknown and arbitrary covariance structure and the rates are naturally adaptive to the true signal. 
\end{itemize}

Our findings unveil new insights into the adversarially robust sample complexity of the conditional Gaussian model which goes beyond of what the current theory has to offer.  

\subsection{Other related works}
The conditional Gaussian models or mixture of Gaussians has been studied a lot in statistics and computer science literature. 
An incomplete and more recent list includes \citet{kim2006robust,azizyan2013minimax,li2015fast,li2017minimax,tony2019high}. 
In the context of adversarial robustness, since the seminal work of \cite{schmidt2018adversarially}, there are several other papers that studied the sample complexity issue in conditional Gaussian models. \citet{NIPS2019_8968} also provided a slightly improved bound in the same setting.  \citet{carmon2019unlabeled}, \citet{stanforth2019labels}, \citet{zhai2019adversarially} showed that with the help of unlabeled data, it is possible to achieve high robust accuracy with the same number of labeled data required for standard learning.  

Another line of research study the sample complexity of adversarially robust learning under the PAC framework, using extensions of Rademacher complexity or VC dimension, including \citet{attias2018improved}, \citet{khim2018adversarial}, \citet{yin2018rademacher}, \citet{cullina2018pac}, \citet{montasser2019vc}, \citet{awasthi2020adversarial}.
The tradeoff in standard and robust accuracy has been theoretically and empirically studied in \citet{zhang2019theoretically}, \citet{suggala2018revisiting},  \citet{tsipras2018robustness}, \citet{raghunathan2020understanding} and \citet{javanmard2020precise}.

Several previous works analyzed the robustness of specific family of classifiers. The early work of \citet{xu2009robust, xu2009robustness} estabilished the connections between robust optimization for linear models and certain types of regularization in classification and regression settings. Subsequently, \citet{xu2012robustness} also showed that under certain notion of robustness, robust algorithms can generalize well. \citet{wang2017analyzing} studied the robustness of nearest neighbor classifiers. 

From the aspect of computational complexity, some recent works  showed that learning a robust model or even verifying robustness of a given model can be computationally hard, including \citep{bubeck2018a, bubeck2018b} and  \citep{awasthi2019robustness, weng2018towards}.

\subsection{Notations}

For the reader's convenience, we list here our notational conventions. 

For positive semi-definite matrix $A$, we use $\|x\|_A \defn \sqrt{x^T A x}$.
Let $\Phi(\cdot)$ the CDF of standard Gaussian distribution $\mathcal{N}(0, 1)$ and 
$\bar{\Phi}(x) \defn 1 - \Phi(x)$. 
The notation $f(n,d) = O\left(g(n, d)\right)$ means that there exits a universal constant $c>0$ that does not depend on the problem parameters such as $n,d$ etc, 
such that $|f(n,d)|\leq c|g(n,d)|$. Similarly, we define $f(n,d) = \Omega\left(g(n,d)\right)$ when there exist constants $c_{1},c_{2}>0$ such that $c_{1}|g(n,d)|\leq|f(n,d)|\leq c_{2}|g(n,d)|$. 
Notation $O_P, \Omega_P$ are used if the corresponding relations happen with probability converges to 1 as 
$n \to \infty$ (see e.g. Chapter 2 of \citep{van2000asymptotic}).
We define the $\ell_{p}$ norm $\|x\|_p = (\sum_{i=1}^d x^p_{i})^{1/p}$ and the corresponding $\ell_{p}$-ball as $\{x \in \R^d |\|x\|_p\leq 1\}$.

\section{Preliminaries}

This section is devoted to setting up the adversarial robust classification problem that is considered in this paper. Along the way, we introduce necessary background and state several preliminary results for future comparisons. 

\paragraph{Conditional Gaussian Model} 
We consider the binary classification problem with data pair $(x,y)$ generated from the mixture of two Gaussian distributions $P_{\mu,\Sigma}$,
\begin{align*}
p(y=1) &= \half, \quad p(y=-1)= \half,\\
p(x|y) &= \mathcal{N}(x;y \mu, \Sigma).
\end{align*}
Here $\mu \in \R^d$,  $\Sigma \in \R^{d\times d}, \Sigma \succeq 0$ denote the mean and covariance of the Gaussian distribution.
Given $n$ training samples $(x_i, y_i) \sim_{i.i.d.} P_{\mu, \Sigma}$ for $1 \leq i \leq n$, the goal is to learn a classifier $\hat{f}(x)$ for predicting the class of a future data point that is drawn from the same distribution $P_{\mu, \Sigma}$.

\paragraph{Adversarially Robust Classification} 

In the standard setting of classification, the optimal classifier is defined as the one that which minimizes the population classification error 
\begin{equation*}\label{def:std_err}
\rstd_{\mu, \Sigma}(f) :=  \E_{(x,y) \sim P_{\mu, \Sigma}} \left[ \mathbb{I}(f(x) \neq y) \right].
\end{equation*}
which we refer to the standard error throughout. 
In this paper, we consider the classification problem under conditional Gaussian generative model in presence of an adversary --- which is to say --- at the testing stage, an adversary is allowed to add any perturbation $\delta$ to the input $x$, that has bounded magnitude $\|\delta\|_B \leq \varepsilon$.
The norm defined here is the standard Minkowski functional that associated with a convex set \cite{thompson1996minkowski}. Formally, given a closed and origin-symmetric convex set $B$, 
the Minkowski functional is defined as 
\begin{align*}
\|x\|_B := \inf\{\lambda \in \mathbb {R}_{>0}: x\in \lambda B\}.
\end{align*}
For instance, when $B$ is the $\ell_p$ unit ball, then $\|x\|_B$ boils down to the classical $\ell_p$ norm of $x$.
In practice, the most widely considered norm for the adversary are $\ell_\infty$ and $\ell_2$ norms.

In the adversarially robust setting, a mapping $f: \R^d \rightarrow \{-1, +1\}$ classifies a sample $(x,y)$ correctly, if and only if the prediction agrees with the true label for \emph{all} possible perturbations of the adversary. To put it in mathematical form, 
\begin{equation*}
\ell_{B, \varepsilon}(f; x, y) := \mathbb{I}\left( \exists \delta: \|\delta\|_B \leq \varepsilon, ~f(x+\delta) \neq y\right).
\end{equation*}

Our goal is to obtain a classifier with minimal expected robust classification error,  i.e. 
finding mapping $f$ that minimizes  
\begin{align}
\label{def:robust_error}
\notag	\RobustRisk(f) &= \E_{(x,y) \sim P_{\mu, \Sigma}} [\ell_{B, \varepsilon}(f; x, y)]\\
&= \E_{(x,y) \sim P_{\mu, \Sigma}} [ \mathbb{I}\left( \exists \|\delta\|_B 
\leq \varepsilon, f(x+\delta) \neq y\right)].
\end{align}
The optimal risk is then defined as the classification error regarding the optimal classifier, namely 
\begin{align}
\RobustOptRisk := \RobustRisk(f_*),
\end{align}
and accordingly, we define the excess risk of any classifier $f$ as
\begin{align}
\RobustRisk(f) - \RobustOptRisk,
\end{align}
which by definition is always non-negative.

\paragraph{Robust Bayes Optimal Classifier} 
To motivate the robust optimal classifiers, we start our discussion with the optimal risk and optimal classifier in the conditional Gaussian Model. 
We note that when $\varepsilon=0$, i.e. there is no adversary, the classification problem reduces to the well-known \emph{Fisher's Linear Discriminant Analysis} problem, where the Bayes optimal classifier is a simple linear classifier
\begin{align*}
f_{\text{Bayes} }(x) = \sgn(\mu^Tx),
\end{align*}
known as Fisher's linear discriminant rule (see, e.g.~\citet{johnson2002applied}). 
The Bayes optimal classifier minimizes the misclassification rate. However, the classifier that minimizes the \emph{robust} classification error is not known until recently, where \cite{NIPS2019_8968} provided a tight lower bound on the minimal robust classification error via optimal transport techniques. 
It is also proved that the optimal risk can be written as the optimal value of a convex program, and the \emph{oracle} optimal classifier is a linear classifier that has a closed form given the solution of the convex program. 

We find it is useful to first simplify and restate this result in order to set the stage for our main result. 
\begin{theorem}[Restated and simplified from \citet{NIPS2019_8968}]
	\label{thm:robust_bayes_optimality}
	Let $\Zpop$ be the solution of the following convex program:
	\begin{equation}\label{eqn:z0}
	\Zpop = \argmin_{\|z\|_B \leq \varepsilon} \|\mu - z\|_{\Sigma^{-1}}^2 ,
	\end{equation}
	where $\|x\|_A = \sqrt{x^T A x}$. \footnote{Note that this notation is different with \cite{NIPS2019_8968}, where in their notation $\|x\|_A = \sqrt{x^T A^{-1} x}$.}Then, the optimal robust classifier for $P_{\mu, \Sigma}$ is a linear classifier $f_*(x) = \sgn(w_0^T x)$, where 
	\begin{equation}\label{eqn:def_w_0}
	w_0 := \Sigma^{-1} (\mu - z_{\Sigma}(\mu)),
	\end{equation}
	and the optimal robust classification error is
	\begin{equation*}
	\RobustOptRisk := \bar{\Phi}(\|w_0\|_\Sigma) = \bar{\Phi}(\|\mu - \Zpop\|_{\Sigma^{-1}}).
	\end{equation*}	
\end{theorem}
\noindent We remark that the above mentioned classifier is indeed an oracle classifier since it is constructed using the unknown parameters $\mu$ and $\Sigma.$

\paragraph{Adversarial Signal-To-Noise Ratio (AdvSNR).} 
In the context of standard classification in the conditional Gaussian model, the notion of Signal-To-Noise Ratio was introduced to measure the effective separation which is defined as the Mahalanobis distance between the means of two conditional distributions.

\begin{definition}[Standard Signal-To-Noise Ratio] The Standard Signal-To-Noise Ratio (StdSNR) of  conditional Gaussian model $P_{\mu,\Sigma}$ is defined as
	\begin{equation*}
	\StdSNR(\mu, \Sigma) \defn 2 \|\mu\|_{\Sigma^{-1}}.
	\end{equation*}	
\end{definition}

Here, the constant $2$ is introduced to be consistent with the literature in Fisher's LDA, e.g. \cite{tony2019high}, where SNR is defined as the Mahalanobis distance between means of two mixture components. We make the note that the StdSNR measures the difficulty of standard classification in the conditional Gaussian model, since the minimal misclassification error equals to $\bar{\Phi}(\half \StdSNR(\mu, \Sigma))$ \cite{tony2019high}. 
In fact, the misclassification error decreases exponentially as the StdSNR increases.  

When it comes to the adversarial setting, StdSNR, however, is no longer a proper metric for the classification difficulty. 
Specifically, conditional Gaussian models with the same StdSNR can have very different levels of hardness in the adversarially robust classification problem. 
In order to illustrate this, we demonstrate a simple example. 

\begin{example}
	Consider an adversary which is allowed to perturb the input with budget 
	$\varepsilon = \frac{6}{\sqrt{d}}$ in terms the $\ell_\infty$ norm. 
	Set the covariance $\Sigma$ to be the identity matrix $I_d$. 
	We examine two conditional Gaussian models, $P_{\mu_1, \Sigma}$ and  $P_{\mu_2, \Sigma}$ with different means $\mu_1$ and $\mu_2$, where
	\begin{align*}
	\mu_1 = \frac{6}{\sqrt{d}} \cdot (1,1,1, \cdots, 1)^T, 
	\quad
	\mu_2 = (6,0,0, \cdots, 0)^T.
	\end{align*}{}
\end{example}{}

It is easily seen that $\|\mu_1\|_{\Sigma^{-1}} = \|\mu_2\|_{\Sigma^{-1}} = 6$, therefore $P_{\mu_1, \Sigma}$ and $P_{\mu_2, \Sigma}$ have the same StdSNR. However, by Theorem \ref{thm:robust_bayes_optimality}, these two distributions actually exhibit completely different minimal robust classification error, indeed, 
\begin{align*}
R_{\mu_1, \Sigma}^{B, \varepsilon} = \bar{\Phi}(0) = \half, 
\quad
R_{\mu_2, \Sigma}^{B, \varepsilon} = \bar{\Phi}(6 - \frac{6}{\sqrt{d}}).
\end{align*}
When the dimension $d$ is sufficiently large, the optimal risk $R_{\mu_2, \Sigma}^{B, \varepsilon}$ approaches $\bar{\Phi}(6) \approx 10^{-8}$, which means there exists a very good robust classifier for $P_{\mu_2, \Sigma}$. 
In contrast, the optimal risk $R_{\mu_1, \Sigma}^{B, \varepsilon} = \half$, i.e. no classifier can achieve a robust accuracy better than a uninformative predictor that classifies everything as the same class. 
From this simple example, it is safe to conclude that StdSNR is not an ideal measurement for the difficulty in the adversarially robust classification problem.

To address the above issue, one need a proper definition of the signal-to-noise-ratio that is suitable for the adversarial robust setting. 
Therefore we introduce the Adversarial Signal-To-Noise Ratio (AdvSNR) for 
any $(B, \varepsilon)$ adversary.

\begin{definition}[Adversarial Signal-To-Noise Ratio] Define the $(B, \varepsilon)$ Adversarial Signal-To-Noise Ratio (AdvSNR) of conditional Gaussian model $P_{\mu,\Sigma}$ as 
	\begin{equation*}
	\AdvSNR_{B, \varepsilon}(\mu, \Sigma) := 2\|\mu - \Zpop\|_{\Sigma^{-1}} = 2 \|w_0\|_\Sigma,
	\end{equation*}	
	where $w_0$ is defined in \eqref{eqn:def_w_0}.
\end{definition}

As a consequence of Theorem \ref{thm:robust_bayes_optimality}, the minimal robust classification error 
satisfies 
\begin{align}
\label{EqnOptRisk}
\RobustOptRisk = \bar{\Phi}\left(\half \AdvSNR(\mu, \Sigma)\right).
\end{align}
Consequently, the AdvSNR fully characterizes the difficulty for the adversarially robust setting as the StdSNR in the standard setting. We also note that when $\varepsilon=0$, i.e. there is no adversary, the AdvSNR reduces to the traditional definition of the StdSNR. Thus, AdvSNR is a reasonable generalization for StdSNR.

Naturally, for every $r>0$, one can consider a class of distributions where each of them has the same $(B, \varepsilon)$-AdvSNR equal to $r$. Within each class, they should enjoy the same hardness of the classification problem. Formally, let us define the class $D_{B, \varepsilon}(r)$.
\begin{definition}\label{def:family_advsnr}
	The family of conditional Gaussian models with $(B, \varepsilon)$-AdvSNR value of $r$, is defined as:
	\begin{equation*}
	D_{B, \varepsilon}(r) \defn \{(\mu, \Sigma)| \AdvSNR_{B, \varepsilon}(\mu, \Sigma) = r\}.
	\end{equation*}
\end{definition}
In the sequel, we develop our minimax lower bounds over these classes of distributions.  
To assist our analysis, we also define the family of conditional Gaussian models with a standard SNR value of $r$ similarly. 
\begin{definition}
	The family of conditional Gaussian models with a standard SNR value of $r$, is defined as:
	\begin{equation*}
	D_{\mathrm{std}}(r) \defn \{(\mu, \Sigma)| \StdSNR(\mu, \Sigma) = r\}.
	\end{equation*}
\end{definition}
In the derivations of our upper bounds and minimax lower bounds, we make the assumption that the AdvSNR $r$ is strictly bounded away from zero by a universal constant \footnote{for instance, $r\geq 10^{-9}$}, otherwise as a result of Theorem \ref{thm:robust_bayes_optimality}, no classifier can achieve accuracy much better than $\half$, the robust risk of a constant classifier $f(x) \equiv 1$.

\section{A Coputationally Efficient Estimator and Risk Upper Bound}

Thus far, we introduce the notion of $\AdvSNR$ which is known to characterize the minimal robust classification error as in expression~\eqref{EqnOptRisk}.
However, whether there exists a computation-efficient classifier that behaves similarly to the oracle best classifier is still unclear.

This section, we aim to answer this question in the affirmative by constructing such a classifier. 
For the classifier that we shall define in the sequel, we give an exact characterization of its excess 
robust classification error compared with the oracle best classifier. 
Motivated by the fact that the optimal robust classifier has the form of \eqref{eqn:def_w_0}, we design a "plug-in" estimator for $w_0$. The estimator is described in the following algorithm.

\begin{algorithm}[H]
	\caption{A plug-in estimator of $w_0$}
	\label{alg:estimator}
	\begin{algorithmic}
		\STATE {\bfseries Input:} Data pairs $\{(x_i, y_i)\}_{i=1}^n$.
		\STATE {\bfseries Output:} $\widehat{w}$.
		\STATE {\bfseries Step 1:} Define $\widehat{\mu}$ and $\widehat{\Sigma}$ as
		\vspace*{-0.2cm}
		\begin{align*}
		\widehat{\mu} &\defn \frac{1}{n} \sum_{i=1}^n y_i x_i, \qquad 
		\widehat{\Sigma} \defn \frac{1}{n} \sum_{i=1}^n x_i x_i ^T - \widehat{\mu} \widehat{\mu}^T.
		\end{align*}
		 \vspace*{-0.2cm}
		\STATE {\bfseries Step 2:} Solve for $\Zemp$ in the following  
		\begin{equation*}
		\Zemp \defn z_{\widehat{\Sigma}}(\widehat{\mu}) 
		= \argmin_{\|z\|_B \leq \varepsilon} \|\widehat{\mu} - z\|_{\widehat{\Sigma}^{-1}}^2.
		\end{equation*}
		 \vspace*{-0.2cm}
		\STATE {\bfseries Step 3:} Define $\widehat{w} \defn \widehat{\Sigma}^{-1} (\widehat{\mu} -\Zemp)$.
	\end{algorithmic}
\end{algorithm}

The main theorem of this section is to characterize the excess risk bound of the classifier
induced by $\widehat{w}$.

\begin{theorem}\label{thm:main_upper_bound}
	For the $(\|\cdot\|_B, \varepsilon)$ adversary, suppose the adversarial signal-to-noise ratio \\ $\AdvSNR_{B, \varepsilon}(\mu, \Sigma) = r$, then the excess risk of $f_{\widehat{w}}$ is upper bounded by
	\begin{equation*}
	\RobustRisk(f_{\widehat{w}}) - \RobustOptRisk \leq O_P\Big(e^{-\frac{1}{8} r^2} \cdot r \cdot \frac{d}{n}\Big).
	\end{equation*}
\end{theorem}

We take a moment to make several remarks. 
First recall that the $\AdvSNR$ is defined as a measurement for the hardness of the classification problem. Indeed, as the above result shows, the excess risk vanishes exponentially with the $\AdvSNR$. 
Moreover, our estimator is \emph{adaptive} in the sense that it does not require knowing any information about the value of $r$, but the theoretical guarantee improves automatically with larger AdvSNRs. We also note that the dependency with sample size $n$ is $O\left(\frac{1}{n}\right)$, which is the same as the rate of Fisher's LDA, but faster than the typical $O\left(\frac{1}{\sqrt{n}}\right)$ rate.

\paragraph{Comparisons to \cite{schmidt2018adversarially}}

We note that our result generalizes the one showed in \cite{schmidt2018adversarially} in many different aspects: 
\begin{enumerate}
	\item In terms of the perturbations,  \citet{schmidt2018adversarially} considered  
	perturbations in $\ell_\infty$ balls, while ours allow for any convex, closed and origin-symmentric perturbaion set $B$, including all $\ell_p$ balls for $p\geq 1.$
	
	\item Our upper and lower bounds hold for both spherical and non-spherical Gaussians, without 
	the knowledge of the population covariance structure.
	
	\item We impose no restrictions on the separation between Gaussian distributions. \citet{schmidt2018adversarially} studied a very specific regime, where the budget of $\ell_\infty$ adversary is bounded by  $\frac{1}{4}$ , the separation between the means of two Gaussians is $\sqrt{d}$, and the spherical covariance matrix $\Sigma = \sigma^2 I$ satisfies $\sigma \leq \frac{1}{32}d^{1/4}$. This regime is low-noise by design, while our analysis applies to any regime whenever there exists a classifier with robust accuracy slightly better than $\half$ .
	
	\item  Our estimator is consistent, i.e. the excess risk converges to zero as sample size $n \rightarrow \infty$.  The classifier used in \citet{schmidt2018adversarially} is actually $\sgn(\hat{\mu}^T x)$. While this classifier achieve near-optimal classification error in the regime of their interst (the low noise regime mentioned above with Gaussian prior on $\mu$), the excess risk does not converge to zero in general. This is due to the fact that the large-sample limit of their classifier is actually $\sgn(\mu^T x)$, i.e. the Bayes optimal classifier for the standard setting. As we can see from Theorem \ref{thm:robust_bayes_optimality} and a simple simulation in Figure \ref{fig:exp}, the excess risk of their algorithm  saturates at a level above zero, which is very different from the behavior of Algorithm \ref{alg:estimator}.
\end{enumerate}

\begin{figure}
	\centering
	\includegraphics[width=\linewidth]{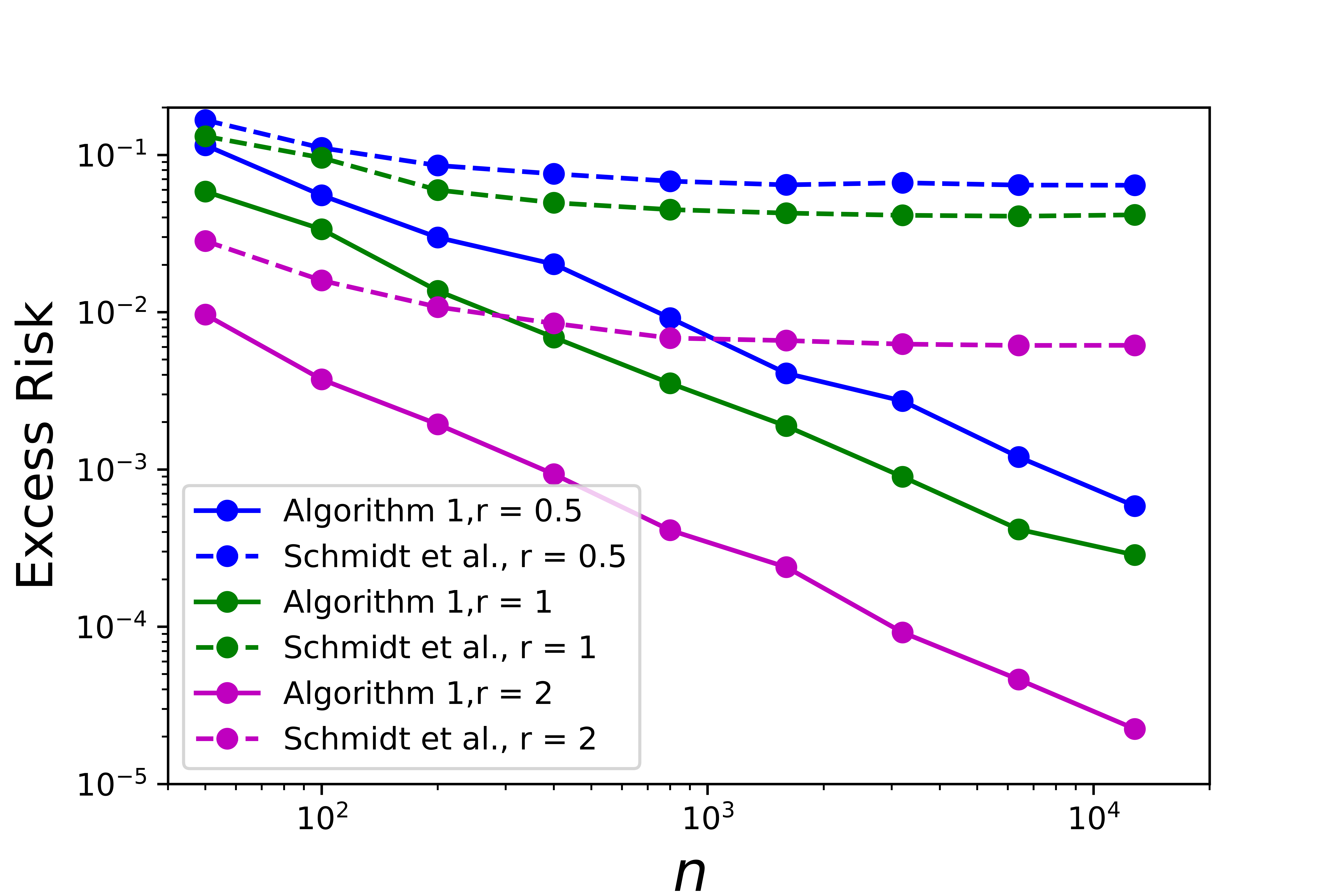}
	\caption{
		A simple simulation on the performance of Algorithm \ref{alg:estimator}  and the algorithm proposed in \citep{schmidt2018adversarially} is shown here with different values of AdvSNR $r$. Here we consider a $50$-dimensional example under $\ell_\infty$ adversary with $\varepsilon=0.1$. The covariance matrix is fixed to be $\Sigma= I$, and the mean parameter $\mu$ is set as $\mu =  (r+\varepsilon, \varepsilon, \varepsilon, \cdots, \varepsilon)  $  for $r \in \{0.5, 1.0, 2.0\}$. We evaluate the excess risk $\RobustRisk(f_{\widehat{w}}) - \RobustOptRisk$ returned by the two algorithms using $n$ i.i.d. training data pairs, where $n \in 
		\{50, 100, 200, 400, 800, 1600, 3200, 6400, 12800\}$. For each combination of $(n, r)$, the averaged excess risk over $10$ random repetitions is reported respectively.}                  
	\label{fig:exp}
\end{figure}
\paragraph{Proof Sketch:} Here we provide a brief sketch of the proof. More details can be found in the Section~\ref{Sec:details}.

\paragraph{Step 1: First order approximation of the risk.} Since both the learned $f_{\hat{w}}$ and the optimal robust classifier $f_*$ are linear classifiers, we can calculate the robust excess risk in closed form using Lemma \ref{lem:error_linear} (also shown in \cite{NIPS2019_8968}):
\begin{equation*}
\RobustRisk(f_{\widehat{w}}) - \RobustOptRisk = \bar{\Phi} \left(\frac{\widehat{w}^T \mu -  \varepsilon \|\widehat{w}\|_{B*}}{\|\widehat{w}\|_\Sigma} \right) - \bar{\Phi}(\half r).
\end{equation*}
By the Taylor expansion of $\bar{\Phi}(\cdot)$, we have
\begin{equation*}
\bar{\Phi} \left(\frac{\widehat{w}^T \mu -  \varepsilon \|\widehat{w}\|_{B*}}{\|\widehat{w}\|_\Sigma} \right) - \bar{\Phi}(\half r) \approx \frac{1}{\sqrt{2\pi}}e^{-\frac{1}{8} r^2} \delta_n, 
\end{equation*}
where 
\begin{equation*}
\delta_n = \half r - \frac{\widehat{w}^T \mu -  \varepsilon \|\widehat{w}\|_{B*}}{\|\widehat{w}\|_\Sigma}.
\end{equation*}
Therefore, it is sufficient to show that $\delta_n = O_P(r \cdot \frac{d}{n})$.

\paragraph{Step 2: Controlling $\delta_n$.} To give an upper bound of $\delta_n$, we will use the fact that  sample mean $\widehat{\mu}$ and sample covariace $\widehat{\Sigma}$ converge to $\mu$ and $\Sigma$ respectively. Furthermore,  the convergence rate is well known as $O_P(\sqrt{\frac{d}{n}})$. 

From a high level, the upper bound of $\delta_n$ is estabished (see Lemma \ref{lem:delta_n}) by carefully decomposing $\delta_n$ into four terms and each 
term is in the form of the differences between population and sample quantities like $\Sigma$ vs $\widehat{\Sigma}$, $\widehat{\mu}$ vs $\mu$. Invoking the convergence rates of $\widehat{\mu}$ and $\widehat{\Sigma}$, we are able to bound each of these terms and complete the proof.

\section{Minimax Lower Bounds}

This section is dedicated to developing minimax excess risk lower bounds for the adversarially robust classification with conditional Gaussian models. 

As is mentioned above, we consider a class of distributions $D_{B, \varepsilon}(r)$ which have the same $\AdvSNR_{B, \varepsilon} = r$, as in Definition~\ref{def:family_advsnr}.
As quantity $\AdvSNR$ characterizes the minimal robust classification error, this class of distributions $D_{B, \varepsilon}(r)$ all share the same adversarially robust classification error. 
Therefore, our lower bounds here measure the fundamental information-theoretic limit of this problem,
namely, no estimator can achieve an essential improvement in terms of the adversarial classification error. 

\begin{theorem}\label{thm:main_lower_bound}
	Let $\widehat{f}$ be any estimator based on $n$ samples $(x_1,y_1), \cdots, (x_n, y_n) \sim_{i.i.d.} P_{\mu, \Sigma}$. We have the following lower bound on the minimax excess risk:	
	\begin{equation*}
	\min_{\widehat{f}} \max_{(\mu, \Sigma) \in D_{B, \varepsilon}(r)} [\RobustRisk(\widehat{f}) - \RobustOptRisk ]  \geq  \Omega_P \Big(e^{-(\frac{1}{8} + o(1))r^2} \frac{d}{n} \Big).
	\end{equation*}	
\end{theorem}

Putting together with the upper bound in Theorem~\ref{thm:main_upper_bound}, 
this lower bound matches almost exactly with the upper bound in the regime of interest, therefore they are both optimal up to lower order terms. 

The main technique we used for this lower bound is with a flavor of black-box reduction. In particular, we show that the minimax \emph{robust} excess risk in $D_{B, \varepsilon}(r)$ cannot be smaller than the minimax \emph{standard} excess risk in $D_{\mathrm{std}}(r)$. In other words, 

\begin{lemma}\label{thm:minimax_reduction}
	The minimax excess error satisfies 
	\begin{equation}\label{eqn:minimax_reduction}
	\min_{\widehat{f}} \max_{(\mu, \Sigma) \in D_{B, \varepsilon}(r)} [\RobustRisk(\widehat{f}) - \RobustOptRisk ] \nonumber  \\
	\geq~  \min_{\widehat{f}} \max_{(\mu', \Sigma) \in D_{\mathrm{std}}(r)}  [\rstd_{\mu', \Sigma}(\widehat{f}) -\CleanOptRisk].
	\end{equation}	
\end{lemma}
\vspace{-0.3cm}
The right hand side of \eqref{eqn:minimax_reduction}, i.e. the minimax rate for standard classification, is well-studied in the existing literature of Fisher's LDA. For example, \cite{li2017minimax} proved the following lower bound:

\begin{theorem}[Theorem 1 of \cite{li2017minimax}] 
	Suppose the covariance matrix satisfies $\Sigma =I$ and is known to the learner, then we have the minimax lower bound 
	\begin{equation*}
	\min_{\widehat{f}} \max_{(\mu', I) \in D_{\mathrm{std}}(r)}  [\rstd_{\mu', \Sigma}(\widehat{f}) -\CleanOptRisk] \\
	\ge \Omega_P\left(e^{-\frac{1}{8} r^2}\cdot \frac{1}{r} \cdot \frac{d}{n} \right).
	\end{equation*}	
\end{theorem}
\vspace{-0.3cm}
Since the parameter space considered in \cite{li2017minimax} is a subset of $D_{\mathrm{std}}(r)$, we have \eqref{eqn:minimax_reduction} is also lower bounded by $\Omega_P\left(e^{-\frac{1}{8} r^2}\cdot \frac{1}{r} \cdot \frac{d}{n} \right)$, therefore proves Theorem \ref{thm:main_lower_bound}.

\paragraph{Comparisons to \cite{schmidt2018adversarially} and \cite{NIPS2019_8968}} 

To the best of our knowlege, Theorem \ref{thm:main_lower_bound} is the first minimax-type lower bound in adversarially robust classification. Existing works \cite{schmidt2018adversarially} and \cite{NIPS2019_8968} also studied the sample complexity of robust learning in conditional Gaussian model. However, both of them simplified the problem and considered the case when $\mu$ follows from a prior distribution $\mathcal{N}(0, I)$. This assumption is crucial to their analysis, otherwise the posterior distribution of $\mu$ given training data is intractable. Hence, the technical tool used in prior works is not sufficient for developing such a minimax lower bound of our interest.

\paragraph{Proof Sketch:} Here we also provide a proof sketch to Lemma~\ref{thm:minimax_reduction}. More details can be found in the Section~\ref{Sec:details}.
\paragraph{Step 1: Connecting standard and robust risks} In Lemma \ref{lem:lb_robust_error}, we prove that for any classifier $f$ and a perturbed distribution $P_{\mu', \Sigma}$, where $\|\mu' - \mu\|_B \leq \varepsilon$, the robust risk of $f$ on $P_{\mu, \Sigma}$ is always lower bounded by the standard risk on $P_{\mu', \Sigma}$. 

As a consequence, in Corollary \ref{corollary:excess_lb} we show that if we choose $\mu' = \mu - z_{\Sigma}(\mu)$, then the robust excess risk of $f$ on $P_{\mu, \Sigma}$ is always lower bounded by the standard excess risk on $P_{\mu', \Sigma}$.

\paragraph{Step 2: A mapping fron $D_{\mathrm{std}}(r)$ to $ D_{B, \varepsilon}(r)$} To prove Lemma~\ref{thm:minimax_reduction}, we only need to answer the following question: for any $(\mu', \Sigma) \in D_{\mathrm{std}}(r)$, can we find a $(\mu, \Sigma) \in D_{B, \varepsilon}(r)$, so that the robust excess risk on $P_{\mu, \Sigma}$ is always lower bounded by the standard excess risk on $P_{\mu', \Sigma}$? We give an affirmative answer to this question. The proof in a combination of Corollary \ref{corollary:excess_lb} showed in Step 1 and an examination of optimality condition in the optimization problem \ref{eqn:z0}.


\section{Comparing Adversarial and Standard Rates}

Putting the upper and lower bounds together provides a comprehensive view of the statistical aspect of the adversarially robust classification. A key question to ask is that: How much does the classification error blows up as the price of being adversarially robust? 

To answer this question, it is sufficient to compare the optimal risks in both cases. Informally,
one can write the logarithm ratio between two rates as 
\begin{align}\label{eqn:rate_comparison}
&\log\left(\frac{\rm{AdvRate}}{\rm{StdRate}}\right)  \approx  \frac{1}{2}\left( \|\mu - z_{\Sigma}(\mu)\|_{\Sigma^{-1}}^2 -  \|\mu \|_{\Sigma^{-1}}^2\right).
\end{align}

From the definition of $z_{\Sigma}(\mu)$ in \eqref{eqn:z0}, we can see that $\|\mu - z_{\Sigma}(\mu)\|_{\Sigma^{-1}}^2 \leq  \|\mu \|_{\Sigma^{-1}}^2$, hence adversarial rate is always slower. 

To analyze this difference quantitively and interpretably, we consider the special case where $\Sigma = I$ and the adversary is $\ell_2$ bounded. Similar results hold for other adversaries as well. The key observation is that depending on the different scale of $\|\mu\|_2$ and the budget of perturbation $\varepsilon$, this difference can be as small as $O(1)$, or as large as $\Omega(\exp(d))$.

\begin{proposition}
	\label{prof_consequence}
	When $\Sigma =I$ and the adversarial perturbation satisfies $\|\delta\|_2 \leq \varepsilon$, then
	\begin{itemize}
		\item When $\varepsilon \leq O(\frac{1}{\|\mu\|_2})$, the adversarial rate is at most $O(1)$ times slower than the standard rate.
		\item When $\|\mu\|_2 \geq \Omega(\log d)$ and $\varepsilon \geq \Omega(\frac{\log d}{\|\mu\|_2})$, the adversarial rate can be slower than the standard rate by a $poly(d)$ factor. 
		\item  When $\|\mu\|_2 \geq \Omega(\sqrt{d})$ and $\varepsilon \geq \Omega(\frac{d}{\|\mu\|_2})$, the adversarial rate can be slower than the standard rate by an $\exp(d)$ factor. 
	\end{itemize}
\end{proposition}
In general, the difference is more significant when $\varepsilon$ or $\|\mu\|_2$ is larger.
This example demonstrates a clear tradeoff between being adversarial robust and obtaining the optimal
accuracy, in particular in the case of large perturbations.


\section{Proofs and further details}
\label{Sec:details}
In this section, we provide detailed proofs for our main results. 
The proof details of some lemmas are deferred to our supplementary file. 

\subsection{Proof of Theorem~\ref{thm:main_upper_bound}}

Before presenting our analysis, we first state a standard lemma about the convergence of empirical mean and covariance. 

\begin{lemma} [Convergence of the empirical mean and covariance (see, e.g. \citet{wainwright2019high})] \label{lem:mean_convergence}
	The convergence rates of the empirical mean $\widehat{\mu}$ and $\widehat{\Sigma}$ to the corresponding ground truth satisfy 
	\begin{equation*}
	\|\widehat{\mu} - \mu \|_{\Sigma^{-1}} = O_P\left(\sqrt{\frac{d}{n}}\right),
	\end{equation*}
	and
	\begin{equation*}
	\| \Sigma^{-\half} \widehat{\Sigma} \Sigma^{-\half} - I\|_{op} = O_P\left(\sqrt{\frac{d}{n}}\right).
	\end{equation*}
\end{lemma}

The following lemma about the classification error of linear classifiers will also be useful for us.

\begin{lemma} [Robust classification error of linear classifier, (see e.g. in \cite{NIPS2019_8968}, Appendix B.3)] \label{lem:error_linear}For a linear classifier $f_w(x) = \sgn(w^Tx)$, the robust classification error with a $B, \varepsilon$ adversary is 
	\begin{equation*}
	\RobustRisk(f_w) = \bar{\Phi}\left(\frac{w^T \mu - \varepsilon \|w\|_{B*}}{\|w\|_\Sigma}\right).
	\end{equation*}
	Here, $\|\cdot\|_{{B*}}$ is the dual norm of $\|\cdot\|_B$.  We use $\RobustRisk(w)$ as a shorthand for $\RobustRisk({f_{w}})$ when the meaning is clear from context.
\end{lemma}

\begin{proof}[Proof of Theorem \ref{thm:main_upper_bound}]
	By Lemma \ref{lem:error_linear} and Taylor expansion of $\bar{\Phi}(t)$ around $t = \half r =  \|w_0\|_\Sigma$, the excess risk can be written as:
	\begin{align*}
	\RobustRisk(\widehat{w}) - \RobustOptRisk &= \bar{\Phi} \left(\frac{\widehat{w}^T \mu -  \varepsilon \|\widehat{w}\|_{B*}}{\|\widehat{w}\|_\Sigma} \right) - \bar{\Phi}(\|w_0\|_\Sigma) \\
	&=\frac{1}{\sqrt{2\pi}}e^{-\frac{1}{8} r^2} \delta_n + O(\delta_n^2),
	\end{align*}
	where 
	\begin{equation*}
	\delta_n =\|w_0\|_\Sigma - \frac{\widehat{w}^T \mu -  \varepsilon \|\widehat{w}\|_{B*}}{\|\widehat{w}\|_\Sigma}.
	\end{equation*}
	
	Therefore, to analyze the convergence rate of the excess risk, we only need to analyze the convergence rate of $\delta_n$. We would like to prove that
	\begin{equation*}
	\delta_n  = O_P\left(r \cdot \frac{d}{n}\right).
	\end{equation*}
	The following lemma is the key of our analysis: it decomposes $\delta_n$ into four terms, each in the form of the difference between population and sample  quantities like $\Sigma$ vs $\widehat{\Sigma}$, $\widehat{\mu}$ vs $\mu$. 
	\begin{lemma}\label{lem:delta_n} We have the following decomposition for $\delta_n$:
		\begin{equation*}
		\|\widehat{w}\|_\Sigma \delta_n 
		= 
		\underbrace{-\half \left( \|w_0\|_\Sigma - \|\widehat{w}\|_\Sigma\right)^2 }_{\termone}
		\underbrace{+ w_0^T  (\Zemp  - \Zpop )}_{\termtwo} \nonumber\\
		\underbrace{- \half \| \Zemp - \Zpop \|_{\Sigma^{-1}}^2}_{\termthree} 
		\underbrace{+ \half \| (\Sigma - \widehat{\Sigma})\widehat{w}  + (\widehat{\mu} - \mu)\|_{\Sigma^{-1}}^2 }_{\termfour}.
		\end{equation*}
		where $\zhat$ is the shorthand for $\zhat = z_{\widehat{\Sigma}}(\widehat{\mu})$.
	\end{lemma}
	The proof of Lemma \ref{lem:delta_n} is provided in Appendix~\ref{Sec:Pflem:delta_n}. 
	Based on this decomposition, our goal is to establish the following relations.
	\begin{equation*}
	\termone \leq 0, ~\termtwo \leq 0, ~\termthree \leq 0, ~\termfour \leq O_P\left(r^2 \frac{d}{n}\right).
	\end{equation*}
	
	It is obvious that $\termone \leq0, \termthree \leq 0$. For the second term $\termtwo$,  consider $\phi(z) = \|\mu - z\|_{\Sigma^{-1}}^2$. Since $\Zpop = \argmin_{\|z\|_B \leq \varepsilon} \|\mu - z\|_{\Sigma^{-1}}^2 = \argmin_{\|z\|_B \leq \varepsilon} \phi(z)$ , by the first order optimality condition, 
	we have $(z’ - \Zpop)^T \grad \phi(\Zpop)\leq 0$ holds for any $\|z’\|_B \leq \varepsilon$. Choosing $z’ = \hat{z}$ gives:
	\begin{equation*}
	(\mu - \Zpop)^T \Sigma^{-1} (\Zemp -\Zpop) \leq0 \Leftrightarrow w_0^T (\Zemp-\Zpop) \leq 0.
	\end{equation*}
	Therefore, $\termtwo \leq 0$ as we desired. 
	
	The remaining work is to prove that $\termfour \leq O_P\left((1+r)^2 \frac{d}{n}\right)$. 
	By triangle's inequality, 
	\begin{equation*}
	\| (\Sigma - \widehat{\Sigma})\widehat{w}  + (\widehat{\mu} - \mu)\|_{\Sigma^{-1}} \leq \| (\Sigma - \widehat{\Sigma})\widehat{w} \|_{\Sigma^{-1}} + \|\widehat{\mu} - \mu\|_{\Sigma^{-1}}.
	\end{equation*}
	Both terms can be controled using covergence of sample mean and covariance. By Lemma \ref{lem:mean_convergence}, one has 
	\begin{equation*}
	\|\widehat{\mu} - \mu\|_{\Sigma^{-1}} \leq O_P\left(\sqrt{\frac{d}{n}} \right),
	\end{equation*}
	and direct calculations give 
	\begin{align*}
	\| (\Sigma - \widehat{\Sigma})\widehat{w} \|_{\Sigma^{-1}} & = \|(I - \Sigma^{-\half} \widehat{\Sigma} \Sigma^{-\half}) (\Sigma^{\half} \widehat{w})\|_2 \\
	& \leq \| I - \Sigma^{-\half} \widehat{\Sigma} \Sigma^{-\half}\|_{op} \|\Sigma^{\half} \widehat{w}\|_2 
	\\
	& = O_P\left(\sqrt{\frac{d}{n}}  \right) \|\widehat{w}\|_{\Sigma}. 
	\end{align*}	
	Combined pieces together, triangle's inequality further guarantees that 
	\begin{equation*}
	\| (\Sigma - \widehat{\Sigma})\widehat{w}  + (\widehat{\mu} - \mu)\|_{\Sigma^{-1}} \leq O_P\left(\sqrt{\frac{d}{n}} (\|\widehat{w}\|_{\Sigma}+1) \right) .
	\end{equation*}
	Since $\widehat{\mu} \rightarrow \mu, \widehat{\Sigma} \rightarrow \Sigma$, we have $\widehat{w} \rightarrow w_0$, therefore $\|\widehat{w}\|_{\Sigma} = (1+o(1))\|w_0\|_{\Sigma} = (\half+o(1))r$, hence,
	\begin{align*}
	\termfour & = \half \| (\Sigma - \widehat{\Sigma})\widehat{w}  + (\widehat{\mu} - \mu)\|_{\Sigma^{-1}}^2 \\
	& \leq \half \left( O_P(\sqrt{\frac{d}{n}}) (\|\widehat{w}\|_{\Sigma} + 1) \right)^2 \\
	& = O_P\left(r^2 \cdot \frac{d}{n}\right).
	\end{align*}
	Putting things together and recall that $r = \Omega(1)$, we have 
	\begin{equation*}
	\delta_n  = O_P\left(r \cdot \frac{d}{n}\right).
	\end{equation*}	
	Therefore we have completed the proof.  
\end{proof}
\subsection{Proof of Lemma~\ref{thm:minimax_reduction}}
To prove Lemma~\ref{thm:minimax_reduction}, we start with a simple observation: for any classifier $f$, its standard error on any perturbed distribution $P_{\mu', \Sigma}$ is always a lower bound on robust error of the original distribution $P_{\mu, \Sigma}$, as long as the perturbation has bounded $B$-norm $\|\mu' - \mu\|_B \leq \varepsilon$:

\begin{lemma}
	\label{lem:lb_robust_error}
	For any classifier $f: \R^d \rightarrow \{-1,+1\}$ and any $\mu' \in \R^d, \|\mu' - \mu\|_B \leq \varepsilon$
	\begin{equation*}
	\RobustRisk(f) \geq \rstd_{\mu', \Sigma}(f).
	\end{equation*}
\end{lemma}
\begin{proof} 
	By the definition of robust classification error \eqref{def:robust_error}, we can decompose the error into two parts: the error on positive class ($y=1$) and negative class ($y = -1$), namely, 
	\begin{align}
	\nonumber	\RobustRisk(f) =& \E_{(x,y) \sim P_{\mu, \Sigma}} [\mathbb{I}\left( \exists \|\delta\|_B \leq \varepsilon, f(x+\delta) \neq y\right)] \\
	=&  \half \E_{x \sim N(\mu, \Sigma)} [\mathbb{I}\left( \exists \|\delta\|_B \leq \varepsilon, f(x+\delta) \neq 1\right)] +
	\half \E_{x \sim N(-\mu, \Sigma)} [\mathbb{I}\left( \exists \|\delta\|_B \leq \varepsilon, f(x+\delta) \neq -1\right)]. \label{eqn:robust_err_decomposition} 
	\end{align}
	By choosing the adversarial perurbation as $\delta = \mu' - \mu$, we have the error on positive class is lower bounded by:
	\begin{align}
	\label{eqn:lb_positive_class_error}
	\notag 	 \E_{x \sim N(\mu, \Sigma)} [ \mathbb{I}\left( \exists \|\delta\|_B \leq \varepsilon, f(x+\delta) \neq 1\right)]  
	\geq& \E_{x \sim N(\mu, \Sigma)} [ \mathbb{I}\left( f(x-\mu + \mu') \neq 1\right)]\\
	= & \E_{x' \sim N(\mu', \Sigma)} [ \mathbb{I}\left( f(x') \neq 1\right)].
	\end{align}
	Similarly, by choosing $\delta = \mu -\mu' $, we have the error on negative class is lower bounded by:
	\begin{equation}
	\label{eqn:lb_negtive_class_error}
	\notag  \E_{x \sim N(-\mu, \Sigma)} [ \mathbb{I}\left( \exists \|\delta\|_B \leq \varepsilon, f(x+\delta) \neq 1\right)]  \\
	\geq  \E_{x' \sim N(-\mu', \Sigma)} [ \mathbb{I}\left( f(x') \neq -1\right)].
	\end{equation}
	Hence, combining \eqref{eqn:robust_err_decomposition} , \eqref{eqn:lb_positive_class_error} and \eqref{eqn:lb_negtive_class_error}, we get
	\begin{align*}
	\RobustRisk(f) \geq &\half \E_{x' \sim N(\mu', \Sigma)} [\mathbb{I}\left( f(x') \neq 1\right)]+
	\half \E_{x' \sim N(-\mu', \Sigma)} [\mathbb{I}\left(f(x') \neq -1\right)]\\	
	=& \rstd_{\mu', \Sigma}(f),
	\end{align*}
	where the last step is by the definition of standard error \eqref{def:std_err}. Therefore we have completed the proof.
\end{proof}

Next, we show more connections between robust and standard classification. Namely, the robust Bayes classifier of $P_{\mu, \Sigma}$ coincides with the standard Bayes classifier of $P_{\mu- \Zpop, \Sigma}$, as stated in the following Lemma:
\begin{lemma}\label{lem:same_bayes_optimality}
	Let $\Zpop$ be the solution of \eqref{eqn:z0}, then the robust Bayes classifier of $P_{\mu, \Sigma}$,  $f_*(x) = \sgn (w_0^T x)$,   satisfies the following conditions:
	\begin{enumerate}
		\item $\RobustRisk(f_*) =  \rstd_{\mu-\Zpop, \Sigma}(f_*)$.
		\item $f_*$ is the standard Bayes Optimal Classifier of $P_{\mu - \Zpop, \Sigma}$.
	\end{enumerate}
\end{lemma}
\begin{proof}
	Note that by setting $\varepsilon=0$ in Theorem \ref{thm:robust_bayes_optimality}, we get the characterization of the standard Bayes error and Bayes optimal classifier for conditional Gaussian models. Applying this result for the distribution $P_{\mu-\Zpop, \Sigma}$, we have 
	\begin{enumerate}
		\item The standard Bayes Optimal Classifier of $P_{\mu - \Zpop, \Sigma}$ is $\sgn((\mu-\Zpop)^T \Sigma^{-1} x)$, which is exactly $f_*(x)$.
		\item The standard Bayes error of of $P_{\mu - \Zpop, \Sigma}$ is $\bar{\Phi}(\sqrt{(\mu-\Zpop)^T \Sigma^{-1}(\mu-\Zpop)})$, which is exactly $\RobustOptRisk$.
	\end{enumerate}
	Hence we have completed the proof. 
\end{proof}
As a direct consequence of Lemma \ref{lem:lb_robust_error} and Lemma \ref{lem:same_bayes_optimality}, we have the robust excess risk under $P_{\mu, \Sigma}$ is lower bounded by the standard excess risk under $P_{\mu- \Zpop, \Sigma}$:
\begin{corollary}\label{corollary:excess_lb}
	For any classifier $f: \R^d \rightarrow \{-1,+1\}$, 
	\begin{align*}
	\RobustRisk(f) - \RobustOptRisk &\geq \rstd_{\mu-\Zpop, \Sigma}(f) - \rstd_{\mu-\Zpop, \Sigma}(f_*)\\ &= \rstd_{\mu-\Zpop, \Sigma}(f) - \rstd_{\mu-\Zpop, \Sigma}*,
	\end{align*}
	where
	\begin{equation*}
	\CleanOptRisk =  \inf_g  \rstd_{\mu', \Sigma}(g)
	\end{equation*}
	is the optimal standard risk.
\end{corollary}

The last piece of tool needed for proving Lemma~\ref{thm:minimax_reduction} is a mapping from $ D_{\mathrm{std}}(r)$ to $D_{B, \varepsilon}(r)$ that keeps the excess risk non-decreasing. This is established via the following lemma:

\begin{lemma}
	\label{lem:inverse_prox_map}
	For any $(\mu', \Sigma) \in D_{\mathrm{std}}(r)$, there exists $(\mu, \Sigma) \in D_{B, \varepsilon}(r)$, such that $\mu - \Zpop  = \mu'$, here $\Zpop$ is the optimal solution of \eqref{eqn:z0}.
\end{lemma}

\begin{proof}
	The proof is constructive: we choose $\mu = \mu' + \widetilde{z}_{\Sigma}(\mu')$, where $\widetilde{z}_{\Sigma}(\mu')$ is the maximizer of the following convex program (which is maximizing a linear function over a convex set):
	\begin{align}\label{eqn:z1}
	\widetilde{z}_{\Sigma}(\mu') =  \argmax_{\|z\|_B \leq \varepsilon}\mu'^T \Sigma^{-1} z.
	\end{align}
	We want to prove that  $\mu - \Zpop = \mu'$. By our choice of $\mu$, we also have $\mu = \mu' + \widetilde{z}_{\Sigma}(\mu')$. Hence, we only need to prove that 
	\begin{equation*}
	\widetilde{z}_{\Sigma}(\mu') = \Zpop.
	\end{equation*}
	In other words, we only need to show that $\widetilde{z}_{\Sigma}(\mu')$ is the minimizer of \eqref{eqn:z0}. 
	
	Since \eqref{eqn:z0} is a convex program with a strongly convex objective, it suffices to prove the following first order optimality condition holds for any $\forall \|z'\|_B \leq \varepsilon$:
	\begin{equation*}
	(\mu - \widetilde{z}_{\Sigma}(\mu'))^T \Sigma^{-1} (z' -\widetilde{z}_{\Sigma}(\mu')) \leq0 .
	\end{equation*}
	Since $\mu - \widetilde{z}_{\Sigma}(\mu') = \mu'$, the inequality is equivalent to:
	\begin{equation*}
	\mu'^T \Sigma^{-1} z' \leq \mu'^T \Sigma^{-1} \widetilde{z}_{\Sigma}(\mu'),
	\end{equation*}
	which is correct by the definition of $\widetilde{z}_{\Sigma}(\mu')$. Hence we have completed the proof.	
\end{proof}

Equipped with Lemma \ref{lem:inverse_prox_map}, now we can prove the important lemma:
\begin{proof}[Proof of Lemma~\ref{thm:minimax_reduction}]
	By Lemma \ref{lem:inverse_prox_map}, for any $(\mu', \Sigma) \in D_{\mathrm{std}}(r)$, there exists $(\mu, \Sigma) \in D_{B, \varepsilon}(r)$, such that $\mu - \Zpop  = \mu'$, where $\Zpop$ is the optimal solution of \eqref{eqn:z0}. By Corollary \ref{corollary:excess_lb}, we have the following inequality holds for any fixed $\widehat{f}$:
	
	\begin{align*}
	\rstd_{\mu', \Sigma}(\widehat{f}) -\CleanOptRisk \leq  R_{\mu, \Sigma}(\widehat{f}) - \RobustOptRisk
	.	\end{align*}
	
	Therefore, 
	\begin{align*}
	\rstd_{\mu', \Sigma}(\widehat{f}) -\CleanOptRisk \leq \max_{(\mu, \Sigma) \in D_{B, \varepsilon}(r)}  [R_{\mu, \Sigma}(\widehat{f}) - \RobustOptRisk].
	\end{align*}
	holds for all $(\mu, \Sigma) \in D_{B, \varepsilon}(r)$, which means
	\begin{equation*}
	\max_{(\mu, \Sigma) \in D_{B, \varepsilon}(r)} [\RobustRisk(\widehat{f}) - \RobustOptRisk] \nonumber
	\geq \max_{(\mu', \Sigma) \in D_{\mathrm{std}}(r)} [\rstd_{\mu', \Sigma}(\widehat{f}) -\CleanOptRisk].
	\end{equation*}
	Then, taking minimum over $\widehat{f}$ on both sides proves the theorem.
	
\end{proof}


\section*{Acknowledgements}

Y.W. is supported in part by the NSF grant DMS-2015447 and CCF-2007911.  C.D. and P.R. are supported by DARPA via HR00112020006, and NSF via IIS1909816.

The authors would also like to thank Kaizheng Wang for many helpful discussions, Tianle Cai and Justin Khim for pointing us toward the work of \cite{NIPS2019_8968,tony2019high}, and annonymous reviewer for many suggestions about improving the presentation of the paper.

\bibliography{ref}

\begin{thebibliography}{}

\bibitem[Attias et~al., 2018]{attias2018improved}
Attias, I., Kontorovich, A., and Mansour, Y. (2018).
\newblock Improved generalization bounds for robust learning.
\newblock {\em arXiv preprint arXiv:1810.02180}.

\bibitem[Awasthi et~al., 2019]{awasthi2019robustness}
Awasthi, P., Dutta, A., and Vijayaraghavan, A. (2019).
\newblock On robustness to adversarial examples and polynomial optimization.
\newblock In {\em Advances in Neural Information Processing Systems}, pages
  13760--13770.

\bibitem[Awasthi et~al., 2020]{awasthi2020adversarial}
Awasthi, P., Frank, N., and Mohri, M. (2020).
\newblock Adversarial learning guarantees for linear hypotheses and neural
  networks.
\newblock {\em arXiv preprint arXiv:2004.13617}.

\bibitem[Azizyan et~al., 2013]{azizyan2013minimax}
Azizyan, M., Singh, A., and Wasserman, L. (2013).
\newblock Minimax theory for high-dimensional gaussian mixtures with sparse
  mean separation.
\newblock In {\em Advances in Neural Information Processing Systems}, pages
  2139--2147.

\bibitem[Bahdanau et~al., 2014]{bahdanau2014neural}
Bahdanau, D., Cho, K., and Bengio, Y. (2014).
\newblock Neural machine translation by jointly learning to align and
  translate.
\newblock {\em arXiv preprint arXiv:1409.0473}.

\bibitem[Bhagoji et~al., 2019]{NIPS2019_8968}
Bhagoji, A.~N., Cullina, D., and Mittal, P. (2019).
\newblock Lower bounds on adversarial robustness from optimal transport.
\newblock In {\em Advances in Neural Information Processing Systems}, pages
  7496--7508.

\bibitem[Bubeck et~al., 2018a]{bubeck2018a}
Bubeck, S., Lee, Y.~T., Price, E., and Razenshteyn, I. (2018a).
\newblock Adversarial examples from cryptographic pseudo-random generators.
\newblock {\em arXiv preprint arXiv:1811.06418}.

\bibitem[Bubeck et~al., 2018b]{bubeck2018b}
Bubeck, S., Price, E., and Razenshteyn, I. (2018b).
\newblock Adversarial examples from computational constraints.
\newblock {\em arXiv preprint arXiv:1805.10204}.

\bibitem[Cai and Zhang, 2019]{tony2019high}
Cai, T. and Zhang, L. (2019).
\newblock High dimensional linear discriminant analysis: optimality, adaptive
  algorithm and missing data.
\newblock {\em Journal of the Royal Statistical Society: Series B (Statistical
  Methodology)}, 81(4):675--705.

\bibitem[Carmon et~al., 2019]{carmon2019unlabeled}
Carmon, Y., Raghunathan, A., Schmidt, L., Duchi, J.~C., and Liang, P.~S.
  (2019).
\newblock Unlabeled data improves adversarial robustness.
\newblock In {\em Advances in Neural Information Processing Systems}, pages
  11190--11201.

\bibitem[Cullina et~al., 2018]{cullina2018pac}
Cullina, D., Bhagoji, A.~N., and Mittal, P. (2018).
\newblock Pac-learning in the presence of adversaries.
\newblock In {\em Advances in Neural Information Processing Systems}, pages
  230--241.

\bibitem[Goodfellow et~al., 2014]{goodfellow2014explaining}
Goodfellow, I.~J., Shlens, J., and Szegedy, C. (2014).
\newblock Explaining and harnessing adversarial examples.
\newblock {\em arXiv preprint arXiv:1412.6572}.

\bibitem[Javanmard et~al., 2020]{javanmard2020precise}
Javanmard, A., Soltanolkotabi, M., and Hassani, H. (2020).
\newblock Precise tradeoffs in adversarial training for linear regression.
\newblock {\em arXiv preprint arXiv:2002.10477}.

\bibitem[Johnson et~al., 2002]{johnson2002applied}
Johnson, R.~A., Wichern, D.~W., et~al. (2002).
\newblock {\em Applied multivariate statistical analysis}, volume~5.
\newblock Prentice hall Upper Saddle River, NJ.

\bibitem[Khim and Loh, 2018]{khim2018adversarial}
Khim, J. and Loh, P.-L. (2018).
\newblock Adversarial risk bounds for binary classification via function
  transformation.
\newblock {\em arXiv preprint arXiv:1810.09519}, 2.

\bibitem[Kim et~al., 2006]{kim2006robust}
Kim, S.-J., Magnani, A., and Boyd, S. (2006).
\newblock Robust fisher discriminant analysis.
\newblock In {\em Advances in neural information processing systems}, pages
  659--666.

\bibitem[Krizhevsky et~al., 2012]{krizhevsky2012imagenet}
Krizhevsky, A., Sutskever, I., and Hinton, G.~E. (2012).
\newblock Imagenet classification with deep convolutional neural networks.
\newblock In {\em Advances in neural information processing systems}, pages
  1097--1105.

\bibitem[Li et~al., 2015]{li2015fast}
Li, T., Prasad, A., and Ravikumar, P.~K. (2015).
\newblock Fast classification rates for high-dimensional gaussian generative
  models.
\newblock In {\em Advances in Neural Information Processing Systems}, pages
  1054--1062.

\bibitem[Li et~al., 2017]{li2017minimax}
Li, T., Yi, X., Carmanis, C., and Ravikumar, P. (2017).
\newblock Minimax gaussian classification \& clustering.
\newblock In {\em Artificial Intelligence and Statistics}, pages 1--9.

\bibitem[McLachlan and Peel, 2004]{mclachlan2004finite}
McLachlan, G.~J. and Peel, D. (2004).
\newblock {\em Finite mixture models}.
\newblock John Wiley \& Sons.

\bibitem[Montasser et~al., 2019]{montasser2019vc}
Montasser, O., Hanneke, S., and Srebro, N. (2019).
\newblock Vc classes are adversarially robustly learnable, but only improperly.
\newblock {\em arXiv preprint arXiv:1902.04217}.

\bibitem[Papernot et~al., 2016]{papernot2016limitations}
Papernot, N., McDaniel, P., Jha, S., Fredrikson, M., Celik, Z.~B., and Swami,
  A. (2016).
\newblock The limitations of deep learning in adversarial settings.
\newblock In {\em 2016 IEEE European symposium on security and privacy
  (EuroS\&P)}, pages 372--387. IEEE.

\bibitem[Raghunathan et~al., 2020]{raghunathan2020understanding}
Raghunathan, A., Xie, S.~M., Yang, F., Duchi, J., and Liang, P. (2020).
\newblock Understanding and mitigating the tradeoff between robustness and
  accuracy.
\newblock {\em arXiv preprint arXiv:2002.10716}.

\bibitem[Schmidt et~al., 2018]{schmidt2018adversarially}
Schmidt, L., Santurkar, S., Tsipras, D., Talwar, K., and Madry, A. (2018).
\newblock Adversarially robust generalization requires more data.
\newblock In {\em Advances in Neural Information Processing Systems}, pages
  5019--5031.

\bibitem[Silver et~al., 2016]{silver2016mastering}
Silver, D., Huang, A., Maddison, C.~J., Guez, A., Sifre, L., Van Den~Driessche,
  G., Schrittwieser, J., Antonoglou, I., Panneershelvam, V., Lanctot, M.,
  et~al. (2016).
\newblock Mastering the game of go with deep neural networks and tree search.
\newblock {\em nature}, 529(7587):484.

\bibitem[Stanforth et~al., 2019]{stanforth2019labels}
Stanforth, R., Fawzi, A., Kohli, P., et~al. (2019).
\newblock Are labels required for improving adversarial robustness?
\newblock {\em arXiv preprint arXiv:1905.13725}.

\bibitem[Suggala et~al., 2018]{suggala2018revisiting}
Suggala, A.~S., Prasad, A., Nagarajan, V., and Ravikumar, P. (2018).
\newblock Revisiting adversarial risk.
\newblock {\em arXiv preprint arXiv:1806.02924}.

\bibitem[Szegedy et~al., 2013]{szegedy2013intriguing}
Szegedy, C., Zaremba, W., Sutskever, I., Bruna, J., Erhan, D., Goodfellow, I.,
  and Fergus, R. (2013).
\newblock Intriguing properties of neural networks.
\newblock {\em arXiv preprint arXiv:1312.6199}.

\bibitem[Thompson and Thompson, 1996]{thompson1996minkowski}
Thompson, A.~C. and Thompson, A.~C. (1996).
\newblock {\em Minkowski geometry}.
\newblock Cambridge University Press.

\bibitem[Tsipras et~al., 2018]{tsipras2018robustness}
Tsipras, D., Santurkar, S., Engstrom, L., Turner, A., and Madry, A. (2018).
\newblock Robustness may be at odds with accuracy.
\newblock {\em arXiv preprint arXiv:1805.12152}.

\bibitem[Van~der Vaart, 2000]{van2000asymptotic}
Van~der Vaart, A.~W. (2000).
\newblock {\em Asymptotic statistics}, volume~3.
\newblock Cambridge university press.

\bibitem[Wainwright, 2019]{wainwright2019high}
Wainwright, M.~J. (2019).
\newblock {\em High-dimensional statistics: A non-asymptotic viewpoint},
  volume~48.
\newblock Cambridge University Press.

\bibitem[Wang et~al., 2017]{wang2017analyzing}
Wang, Y., Jha, S., and Chaudhuri, K. (2017).
\newblock Analyzing the robustness of nearest neighbors to adversarial
  examples.
\newblock {\em arXiv preprint arXiv:1706.03922}.

\bibitem[Weng et~al., 2018]{weng2018towards}
Weng, T.-W., Zhang, H., Chen, H., Song, Z., Hsieh, C.-J., Boning, D., Dhillon,
  I.~S., and Daniel, L. (2018).
\newblock Towards fast computation of certified robustness for relu networks.
\newblock {\em arXiv preprint arXiv:1804.09699}.

\bibitem[Xu et~al., 2009a]{xu2009robust}
Xu, H., Caramanis, C., and Mannor, S. (2009a).
\newblock Robust regression and lasso.
\newblock In {\em Advances in neural information processing systems}, pages
  1801--1808.

\bibitem[Xu et~al., 2009b]{xu2009robustness}
Xu, H., Caramanis, C., and Mannor, S. (2009b).
\newblock Robustness and regularization of support vector machines.
\newblock {\em Journal of machine learning research}, 10(7).

\bibitem[Xu and Mannor, 2012]{xu2012robustness}
Xu, H. and Mannor, S. (2012).
\newblock Robustness and generalization.
\newblock {\em Machine learning}, 86(3):391--423.

\bibitem[Yin et~al., 2018]{yin2018rademacher}
Yin, D., Ramchandran, K., and Bartlett, P. (2018).
\newblock Rademacher complexity for adversarially robust generalization.
\newblock {\em arXiv preprint arXiv:1810.11914}.

\bibitem[Zhai et~al., 2019]{zhai2019adversarially}
Zhai, R., Cai, T., He, D., Dan, C., He, K., Hopcroft, J., and Wang, L. (2019).
\newblock Adversarially robust generalization just requires more unlabeled
  data.
\newblock {\em arXiv preprint arXiv:1906.00555}.

\bibitem[Zhang et~al., 2019]{zhang2019theoretically}
Zhang, H., Yu, Y., Jiao, J., Xing, E.~P., Ghaoui, L.~E., and Jordan, M.~I.
  (2019).
\newblock Theoretically principled trade-off between robustness and accuracy.
\newblock {\em arXiv preprint arXiv:1901.08573}.

\end{thebibliography}
\bibliographystyle{apalike}
\nocite{}

\newpage
\appendix
\renewcommand{\prox}{\mathrm{prox}}

\section{Proof of Theorem \ref{thm:robust_bayes_optimality}}
For completeness, in this section, we present the proof of Theorem \ref{thm:robust_bayes_optimality}. This result follows from combining Theorem 1, Theorem 2 and Lemma 1 in \cite{NIPS2019_8968}. The proof is mainly a  simplified presentation of their proofs (e.g. without using the language of optimal transport) which make some of their results explicit to interpret for our case (e.g. they did not provide the expression for optimal linear classifier, which is useful to our algorithmic results). 

To start with, let us define $w_1 := \frac{w_0}{\|w_0\|_{\Sigma}} = \frac{\Sigma^{-1}(\mu - \Zpop)}{\|\mu - \Zpop\|_{\Sigma^{-1}}}$ be the normalized version of $w_0$ so that $\|w_1\|_\Sigma = 1$. The following lemma is implicit in \cite{NIPS2019_8968}:

\begin{lemma}
	Suppose we define  
	\begin{equation*}
	G(z,w)= w^T(\mu - z),
	\end{equation*} 	
	then $( \Zpop, w_1)$ is solution of the following minimax optimization problem:
	\begin{equation}\label{eqn:minimax}
	\min_{\|z\|_B \le \varepsilon} \max_{\|w\|_{\Sigma} \leq 1} G(z,w).
	\end{equation}
\end{lemma}
\begin{proof}
	We first show that the optimal value of the inner maximization problem can be written as:
	\begin{equation}\label{eqn:inner_max}
	\max_{\|w\|_{\Sigma} \leq 1} w^T(\mu - z) = \|\mu - z\|_{\Sigma^{-1}},
	\end{equation}
	and the maximum is achieved when 
	\begin{equation}\label{eqn:inner_argmax}
	w = \frac{\Sigma^{-1}(\mu - z)}{\|\mu - z\|_{\Sigma^{-1}}}.
	\end{equation}
	In fact, for any $w$ such that $\|w\|_{\Sigma} \leq 1$, Cauchy-Schwarz inequality gives 
	\begin{align*}
	w^T(\mu - z)  = (\Sigma^{1/2}w)^T \Sigma^{-1/2}(\mu - z) 
	&\leq  \|\Sigma^{1/2}w\|_2 \|\Sigma^{-1/2}(\mu - z)\|_2 \\
	& = \|w\|_{\Sigma} \|\mu - z\|_{\Sigma^{-1}} \\
	& \leq \|\mu - z\|_{\Sigma^{-1}}.
	\end{align*}	
	Furthermore, it is easy to check that the choice $w = \frac{\Sigma^{-1}(\mu - z)}{\|\mu - z\|_{\Sigma^{-1}}}$ directly yields $w^T(\mu - z) = \|\mu - z\|_{\Sigma^{-1}}$ achieving the equality. Therefore we have proved \eqref{eqn:inner_max} and \eqref{eqn:inner_argmax}. 
	
	Using \eqref{eqn:inner_max}, the minimax problem~\eqref{eqn:minimax} therefore simplifies to: 
	\begin{equation*}
	\min_{\|z\|_B \le \varepsilon} \|\mu - z\|_{\Sigma^{-1}}.
	\end{equation*}
	Recall that we define $\Zpop$ (cf. \eqref{eqn:z0}) as 
	\begin{equation*}
	\Zpop = \argmin_{\|z\|_B \leq \varepsilon} \|\mu - z\|_{\Sigma^{-1}}^2,
	\end{equation*}
	which is the optimal solution to this outer minimization problem. Combining with the optimality condition for the inner maximization \eqref{eqn:inner_argmax}, we conclude that $( \Zpop, w_1)$ is solution of the minimax problem \eqref{eqn:minimax} and complete the proof.
\end{proof}

\begin{corollary}\label{corollary:w_and_z} 
	The following relation is satisfied for quantities $w_1$ and $\Zpop$:
	\begin{equation*}
	w_1^T \mu - \varepsilon \|w_1\|_{B*} = \|\mu - \Zpop\|_{\Sigma^{-1}}.
	\end{equation*}
\end{corollary}
\begin{proof}
	Since $G(z,w)$ is linear in both $z$ and $w$ and both constraint sets $\{ \|z\|_B \le \varepsilon \}$ and $\{ \|w\|_{\Sigma} \leq 1\}$ are convex, the minimax problem \eqref{eqn:minimax} satisfies strong duality by Von Neumann's Minimax Theorem. 
	In other words, we can switch the order of the min and max, namely, 
	\begin{equation*}
	\min_{\|z\|_B \le \varepsilon} \max_{\|w\|_{\Sigma} \leq 1} G(z,w) = \max_{\|w\|_{\Sigma} \leq 1}  \min_{\|z\|_B \le \varepsilon} G(z,w), 
	\end{equation*}
	and $(\Zpop, w_1)$ is the solution to both sides. By the stationary condition of the minimax problem, 
	\begin{equation*}
	\Zpop = \argmin_{\|z\|_B \leq \varepsilon} G(z,w_1).
	\end{equation*}
	By the definition of dual norm, we also have
	\begin{equation*}
	\min_{\|z\|_B \leq \varepsilon} G(z,w_1) = \min_{\|z\|_B \leq \varepsilon} w_1^T(\mu - z) = w_1^T \mu - \varepsilon \|w_1\|_{B*}.
	\end{equation*}
	Hence, 
	\begin{equation*}
	\|\mu - \Zpop\|_{\Sigma^{-1}} = G(\Zpop,w_1) = \min_{\|z\|_B \leq \varepsilon} G(z,w_1) =  w_1^T \mu - \varepsilon \|w_1\|_{B*}.
	\end{equation*}
	Thus we completed the proof.
\end{proof}
Now we are ready to prove Theorem \ref{thm:robust_bayes_optimality}.
\begin{proof}[Proof of Theorem \ref{thm:robust_bayes_optimality}]
	The proof can be divided into two parts: 
	\begin{enumerate}
		\item Show that $f_{w_0}$ has robust risk $\RobustRisk(f_{w_0}) = \bar{\Phi}(\|\mu - \Zpop\|_{\Sigma^{-1}})$.
		\item Show that no classifier can achieve robust risk smaller than $\bar{\Phi}(\|\mu - \Zpop\|_{\Sigma^{-1}})$.
	\end{enumerate}
	The first part is a consequence of Corollary \ref{corollary:w_and_z}. In order to see this, we first note that since $w_1$ is a rescaling of $w_0$, the induced linear classifiers are the same, hence,
	\begin{equation*}
	\RobustRisk(f_{w_0}) = \RobustRisk(f_{w_1}) .
	\end{equation*}
	By Lemma \ref{lem:error_linear}, the robust risk of $f_{w_1}$ is
	\begin{equation*}
	\RobustRisk(f_{w_1}) =\bar{\Phi} (\frac{w_1^T \mu - \varepsilon \|w_1\|_{B*}}{\|w_1\|_{\Sigma}}) = \bar{\Phi} (w_1^T \mu - \varepsilon \|w_1\|_{B*}).
	\end{equation*}
	By Corollary \ref{corollary:w_and_z}, 
	\begin{equation*}
	\bar{\Phi} (w_1^T \mu - \varepsilon \|w_1\|_{B*}) =\bar{\Phi}(\|\mu - \Zpop\|_{\Sigma^{-1}}).
	\end{equation*}
	Therefore, we have proved the first part. 
	
	For the second part, we invoke Lemma \ref{lem:lb_robust_error}. By setting $\mu' = \mu - \Zpop$ in Lemma \ref{lem:lb_robust_error}, we have that for any classifier $f$, 
	\begin{equation*}
	\RobustRisk(f) \geq \rstd_{\mu - \Zpop, \Sigma}(f).
	\end{equation*}
	We also know that no classifier can achieve standard risk smaller than the Bayes Risk in $P_{\mu - \Zpop, \Sigma}$. Recall that for a conditional Gaussian kmodel $P_{\mu', \Sigma}$, the standard Bayes Risk is $\bar{\Phi}(\|\mu'\|_{\Sigma^{-1}})$. In other words, for any classifier $f$, we have
	\begin{equation*}
	\rstd_{\mu - \Zpop, \Sigma}(f) \geq \bar{\Phi}(\|\mu - \Zpop\|_{\Sigma^{-1}}).
	\end{equation*}
	Combining the two inequalities, we conclude that 
	\begin{equation}
	\RobustRisk(f) \geq \bar{\Phi}(\|\mu - \Zpop\|_{\Sigma^{-1}})
	\end{equation}
	holds for all classifiers $f$. Therefore, we prove the second part and thus complete the proof.
\end{proof}

\section{Proof of Proposition~\ref{prof_consequence}}
\begin{proof}[Proof of Proposition~\ref{prof_consequence}]
	Recall that the setting of interest here is $\Sigma = I$ and $\|\cdot\|_B$ corresponds to the $\ell_2$ norm. 
	In this setting, we show that $z_{\Sigma}(\mu)$ has a simplified form. 
	In fact, directly invoking 
	\begin{equation*}
	\Zpop = \argmin_{\|z\|_B \leq \varepsilon} \|\mu - z\|_{\Sigma^{-1}}^2 = \argmin_{\|z\|_2 \leq \varepsilon} \|\mu - z\|_{2}^2,
	\end{equation*}
	gives $z_{\Sigma}(\mu) = \min(\varepsilon, \|\mu\|_2) \frac{\mu}{\|\mu\|_2}$, and 
	\begin{equation*}
	\mu - z_{\Sigma}(\mu) = \max(0, \frac{\|\mu\|_2 - \varepsilon}{\|\mu\|_2}) \mu.
	\end{equation*}
	From this expression, we can see that when $\varepsilon > \|\mu\|_2$, the Adversarial Signal-to-Noise Ratio of $P_{\mu, \Sigma}$ is $2\|\mu - z_{\Sigma}(\mu)\|_2 = 0$. Hence, no classifier can achieve accuracy better than $\half$. Below we only consider the case when $\varepsilon < \|\mu\|_2$.
	
	Recall that we want to compare the minimax rate in adversarial and standard setting. As we showed earlier, the minimax rates are $O(\exp(-\half \|\mu - z_{\Sigma}(\mu)\|_2^2)\frac{d}{n})$ and $O(\exp(-\half \|\mu\|_2^2)\frac{d}{n})$ respectively. The ratio between the two quantities equals to:
	\begin{equation}
	\frac{\exp(-\half \|\mu - z_{\Sigma}(\mu)\|_2^2)\frac{d}{n}}{\exp(-\half \|\mu - z_{\Sigma}(\mu)\|_2^2)\frac{d}{n}} = \exp(\half((\|\mu\|_2 - \varepsilon)^2 - \|\mu\|_2^2)) = \exp(\varepsilon \|\mu\|_2 - \half \varepsilon^2).
	\end{equation}
	Since $0 \leq \varepsilon < \|\mu\|_2$, we have 
	\begin{equation*}
	\varepsilon \|\mu\|_2 - \half \varepsilon^2 = \varepsilon(\|\mu\|_2 - \half \varepsilon) 
	\in \left[ \half\varepsilon \|\mu\|_2, \varepsilon \|\mu\|_2 \right].
	\end{equation*}

	Equipped with the above relation, we are in the position of establishing Proposition~\ref{prof_consequence}.
	\begin{itemize}
		\item When $\varepsilon \leq O(\frac{1}{\|\mu\|_2})$, one has 
		\begin{equation*}
		\varepsilon \|\mu\|_2 - \half \varepsilon^2 \leq \varepsilon \|\mu\|_2 \leq O(1),
		\end{equation*}
		thereby, the adversarial rate is at most $\exp(O(1)) = O(1)$ times slower than the standard rate.
		
		\item When $\|\mu\|_2 \geq \Omega(\log d)$ and $\varepsilon \geq \Omega(\frac{\log d}{\|\mu\|_2})$, we conclude 
		\begin{equation*}
		\varepsilon \|\mu\|_2 - \half \varepsilon^2 \geq \half \varepsilon \|\mu\|_2 \geq \Omega(\log d),
		\end{equation*}
		the adversarial rate can be slower than the standard rate by an $\Omega(\exp (\log d)) = \Omega(poly(d))$ factor. 
		
		\item  When $\|\mu\|_2 \geq \Omega(\sqrt{d})$ and $\varepsilon \geq \Omega(\frac{d}{\|\mu\|_2})$, it is guaranteed that 
		\begin{equation*}
		\varepsilon \|\mu\|_2 - \half \varepsilon^2 \geq \half \varepsilon \|\mu\|_2 \geq \Omega(d),
		\end{equation*}
		therefore, the adversarial rate can be slower than the standard rate by an $\Omega(\exp (d))$ factor.
	\end{itemize} 
\end{proof}

\addtocounter{algorithm}{1}

\section{Improved analysis when $\Sigma$ is known}

Meticulous readers may find a tiny gap between our bounds: the upper bound in Theorem \ref{thm:main_upper_bound} is $O_P\left(e^{-\frac{1}{8} r^2}\cdot r \cdot \frac{d}{n} \right)$, while the lower bound above gives $\Omega_P\left(e^{-\frac{1}{8} r^2}\cdot \frac{1}{r} \cdot \frac{d}{n} \right)$. Since the dominant factor is $ e^{-\frac{1}{8} r^2} $ and $r = \Omega(1)$, this difference is only in a lower order term. This gap is due to the fact that \cite{li2017minimax} assumed the covariance matrix $\Sigma$ is known to the learner. In this section, we will prove that under the same assumption, there is a modified version of Algorithm \ref{alg:estimator} that achieves the truly optimal rate which matches the lower bound even with lower order term in $r$. 

The only modification we made in Algorithm \ref{alg:estimator} is to replace the sample covariance matrix by the true covariance $\Sigma$. The modified algorithm is presented below in Algorithm \ref{alg:estimator_modified}. 

\begin{algorithm}[H]
	\caption{An improved estimator for $w_0$ when $\Sigma$ is known}
	\label{alg:estimator_modified}
	\begin{algorithmic}
		\STATE {\bfseries Input:} Data pairs $\{(x_i, y_i)\}_{i=1}^n$.
		\STATE {\bfseries Output:} $\widehat{w}$.
		\STATE {\bfseries Step 1:} Define $\widehat{\mu}$ and $\widehat{\Sigma}$ as
		\begin{align*}
		\widehat{\mu} &\defn \frac{1}{n} \sum_{i=1}^n y_i x_i, 
		\qquad 
		\widehat{\Sigma} \defn \Sigma.
		\end{align*}
		\STATE {\bfseries Step 2:} Solve for $\Zemp$ in the following  
		\begin{equation*}
		\Zemp \defn z_{\widehat{\Sigma}}(\widehat{\mu}) 
		= \argmin_{\|z\|_B \leq \varepsilon} \|\widehat{\mu} - z\|_{\widehat{\Sigma}^{-1}}^2.
		\end{equation*}
		\STATE {\bfseries Step 3:} Define $\widehat{w} \defn \widehat{\Sigma}^{-1} (\widehat{\mu} -\Zemp)$.
	\end{algorithmic}
\end{algorithm}

\begin{theorem}\label{thm:modified_upper_bound}
	For the $(\|\cdot\|_B, \varepsilon)$ adversary, suppose the adversarial signal-to-noise ratio \\ $\AdvSNR_{B, \varepsilon}(\mu, \Sigma) = r$, then the excess risk of $f_{\widehat{w}}$ defined in Algorithm \ref{alg:estimator_modified} is upper bounded by
	\begin{equation*}
	\RobustRisk(f_{\widehat{w}}) - \RobustOptRisk \leq O_P\left(e^{-\frac{1}{8} r^2} \cdot \frac{1}{r} \cdot \frac{d}{n}\right).
	\end{equation*}
\end{theorem}
This improved rate can be proved by some simple modification to the proof of Theorem \ref{thm:main_upper_bound}.

\begin{proof} 
	
	We demonstrate that in this setting, there is a stronger upper bound $\delta_n  = O_P\left(\frac{1}{r} \cdot \frac{d}{n}\right)$ and the rest of proof follows the same as that of Theorem \ref{thm:main_upper_bound}.
	To this end, let us recall that by Lemma \ref{lem:delta_n} and one has the decomposition, 
	\begin{align*}
	\|\widehat{w}\|_\Sigma \delta_n 
	&= 
	\underbrace{-\half \left( \|w_0\|_\Sigma - \|\widehat{w}\|_\Sigma\right)^2 }_{\termone}
	\underbrace{+ w_0^T  (\Zemp  - \Zpop )}_{\termtwo} \underbrace{- \half \| \Zemp - \Zpop \|_{\Sigma^{-1}}^2}_{\termthree} 
	\underbrace{+ \half \| (\Sigma - \widehat{\Sigma})\widehat{w}  + (\widehat{\mu} - \mu)\|_{\Sigma^{-1}}^2 }_{\termfour}.
	\end{align*}
	Similar to the proof of Theorem \ref{thm:main_upper_bound}, we shall establish that
	\begin{equation*}
	\termone \leq 0, ~\termtwo \leq 0, ~\termthree \leq 0, ~\termfour \leq O_P\left( \frac{d}{n}\right).
	\end{equation*}
	Note that the only difference here is that we can now give a tighter upper bound for $\termfour$: $ O_P\left( \frac{d}{n}\right)$ instead of $ O_P\left(r^2 \frac{d}{n}\right)$. 
	
	Since $\Sigma = \hat{\Sigma}$, by Lemma \ref{lem:mean_convergence}, we have 
	\begin{equation}
	\termfour = \half \| (\Sigma - \widehat{\Sigma})\widehat{w}  + (\widehat{\mu} - \mu)\|_{\Sigma^{-1}}^2 = \half \|(\widehat{\mu} - \mu)\|_{\Sigma^{-1}}^2 = O_P\left( \frac{d}{n}\right).
	\end{equation}
	Hence, we have proved that $T_4 = O_P\left( \frac{d}{n} \right) $, and  
	\begin{equation*}
	\delta_n  = O_P\left(\frac{1}{r} \cdot \frac{d}{n}\right).
	\end{equation*}	
	Therefore we have completed the proof.  
\end{proof}

\section{Proof of Lemma \ref{lem:delta_n}}

\label{Sec:Pflem:delta_n}

\begin{proof}[Proof of Lemma \ref{lem:delta_n}] 
	
	Recall that our goal is to establish
	\begin{align}
	\label{eqn:defT}
	\notag \|\widehat{w}\|_\Sigma \delta_n & = \|\widehat{w}\|_\Sigma \|w_0\|_\Sigma - \left(\widehat{w}^T \mu -  \varepsilon \|\widehat{w}\|_{B*} \right) \\
	&= 
	\underbrace{-\half \left( \|w_0\|_\Sigma - \|\widehat{w}\|_\Sigma\right)^2 }_{\termone}
	\underbrace{+ w_0^T  (\Zemp  - \Zpop )}_{\termtwo}
	\underbrace{- \half \| \Zemp - \Zpop \|_{\Sigma^{-1}}^2}_{\termthree} 
	\underbrace{+ \half \| (\Sigma - \widehat{\Sigma})\widehat{w}  + (\widehat{\mu} - \mu)\|_{\Sigma^{-1}}^2 }_{\termfour}.
	\end{align}
	
	Since $\widehat{w} = \widehat{\Sigma}^{-1}(\widehat{\mu} -z_{\widehat{\Sigma}}(\widehat{\mu}))$, by  Theorem \ref{thm:robust_bayes_optimality}, $f_{\widehat{w}}$ is the optimal robust classifier for $P_{\widehat{\mu}, \widehat{\Sigma}}$, therefore, one can observe 
	\begin{equation*}
	\frac{\widehat{w}^T \widehat{\mu} - \varepsilon \|\widehat{w}\|_{B*}}{\|\widehat{w}\|_{\widehat{\Sigma}}} = \|\widehat{w}\|_{\widehat{\Sigma}}.
	\end{equation*}
	Hence, direct calculations yield 
	\begin{align*}
	\|\widehat{w}\|_{\Sigma}\delta_n &= \|w_0\|_{\Sigma} \|\widehat{w}\|_{\Sigma} - \|\widehat{w}\|_{\widehat{\Sigma}}^2 -  \widehat{w}^T (\mu - \widehat{\mu}) \\
	&= \|w_0\|_{\Sigma} \|\widehat{w}\|_{\Sigma} - (\widehat{\mu}-\Zemp)^T \widehat{\Sigma}^{-1} (\widehat{\mu} - \Zemp) + (\widehat{\mu} - \Zemp)^T \widehat{\Sigma}^{-1} ( \widehat{\mu} - \mu) \\
	&= \|w_0\|_{\Sigma} \|\widehat{w}\|_{\Sigma} + \widehat{w}^T (\Zemp - \mu ).
	\end{align*} 
	Now by use of the relation $\mu = \Sigma w_0 + \Zpop$, 
	we can further obtain 
	\begin{align*}	
	\|\widehat{w}\|_{\Sigma}\delta_n&= \|w_0\|_{\Sigma} \|\widehat{w}\|_{\Sigma} + \widehat{w}^T (\Zemp - \Sigma w_0 - \Zpop ) \\
	&= \|w_0\|_{\Sigma} \|\widehat{w}\|_{\Sigma} - \widehat{w}^T \Sigma w_0 +  \widehat{w}^T (\Zemp  - \Zpop ) \\
	&=   -\half \left( \|w_0\|_\Sigma - \|\widehat{w}\|_\Sigma\right)^2 + \half 	\|w_0\|_{\Sigma}^2 + \half  \|\widehat{w}\|_{\Sigma}^2 - \widehat{w}^T \Sigma w_0 +  \widehat{w}^T (\Zemp  - \Zpop ) \\
	&= \termone + \half (\widehat{w}-w_0)^T \Sigma (\widehat{w}-w_0) + w_0^T (\Zemp-\Zpop)  + (\widehat{w} - w_0)^T (\Zemp  - \Zpop ) \\
	& = \termone + \half (\widehat{w}-w_0)^T \Sigma (\widehat{w}-w_0) + \termtwo  + (\widehat{w} - w_0)^T (\Zemp  - \Zpop ),
	\end{align*}
	where the last equality invokes the definitions in expression~\eqref{eqn:defT}.
	To finish the proof, we make the observation about $\Sigma (\widehat{w}-w_0)$ in the following
	\begin{align*}
	\Sigma (\widehat{w}-w_0) &= (\Sigma - \widehat{\Sigma})\widehat{w} + (\widehat{\Sigma}\widehat{w} - \Sigma w_0) \\
	&= \underbrace{(\Sigma - \widehat{\Sigma})\widehat{w}}_{\termfive} + \underbrace{(\widehat{\mu} - \mu)}_{\termsix} - \underbrace{(\Zemp - \Zpop)}_{\termseven}
	:= \termfive+\termsix-\termseven.
	\end{align*}
	Therefore, putting everything together and rearranging terms, it is guaranteed that  
	\begin{align*}
	\|\widehat{w}\|_{\Sigma}\delta_n  & =   \termone + \termtwo +  \half (\widehat{w}-w_0)^T \Sigma (\widehat{w}-w_0) + (\widehat{w} - w_0)^T (\Zemp  - \Zpop ) \\
	&=  \termone + \termtwo + \half (\Sigma (\widehat{w}-w_0))^T \Sigma^{-1} (\Sigma (\widehat{w}-w_0)) + (\Sigma (\widehat{w}-w_0))^T \Sigma^{-1}(\Zemp-\Zpop) \\
	&=  \termone + \termtwo + \half (\termfive+\termsix-\termseven)^T \Sigma^{-1} (\termfive+\termsix-\termseven) + (\termfive+\termsix-\termseven) \Sigma^{-1} \termseven \\
	&=  \termone + \termtwo + \half (\termfive+\termsix-\termseven)^T \Sigma^{-1} (\termfive+\termsix+\termseven) \\
	&=  \termone + \termtwo - \half \termseven^T \Sigma^{-1} \termseven + \half (\termfive+\termsix)^T \Sigma^{-1} (\termfive+\termsix) \\
	&=  \termone + \termtwo + \termthree + \termfour.
	\end{align*}	
	Thus we have finished the proof.
\end{proof}

\end{document}